\documentclass[10pt,twocolumn,letterpaper]{article}

\pdfoutput=1

\usepackage{cvpr}
\usepackage{times}
\usepackage{epsfig}
\usepackage{graphicx}
\usepackage{amsmath}
\usepackage{amssymb}

\usepackage{algorithm}
\usepackage[noend]{algpseudocode}
\usepackage{xfrac} 
\usepackage{tabularx} 
\usepackage{amsthm}
\usepackage{accents} 
\usepackage{mathtools} 

\makeatletter\let\captiontemp\@makecaption\makeatother
\usepackage[font=footnotesize]{subcaption}
\makeatletter\let\@makecaption\captiontemp\makeatother

\usepackage[pagebackref=true,breaklinks=true,letterpaper=true,colorlinks,bookmarks=false]{hyperref}

\cvprfinalcopy 

\def\cvprPaperID{1870} 

\ifcvprfinal\fi

\graphicspath{{./images/}}

\DeclareMathOperator{\shortsin}{s}
\DeclareMathOperator{\shortcos}{c}

\newcommand{\bR}{\mathbf{R}}
\newcommand{\bt}{\mathbf{t}}
\newcommand{\br}{\mathbf{r}}
\newcommand{\bx}{\mathbf{x}}
\newcommand{\by}{\mathbf{y}}
\newcommand{\bp}{\mathbf{p}}

\newcommand{\ubar}[1]{\underaccent{\bar}{#1}}

\makeatletter
\let\OldStatex\Statex
\renewcommand{\Statex}[1][3]{
  \setlength\@tempdima{\algorithmicindent}
  \OldStatex\hskip\dimexpr#1\@tempdima\relax}
\makeatother

\newtheorem{theorem}{Theorem}
\newtheorem{lemma}{Lemma}
\newtheorem{corollary}{Corollary}

\begin{document}

\setlength{\textfloatsep}{6.0pt plus 2.0pt minus 2.0pt}
\setlength{\floatsep}{6.0pt plus 2.0pt minus 2.0pt}

\title{GOGMA: Globally-Optimal Gaussian Mixture Alignment}

\author{Dylan Campbell and Lars Petersson\\
Australian National University\,\,\,\,\,\,\,\,\,\,National ICT Australia (NICTA)%
\thanks{\tiny NICTA is funded by the Australian Government through the Department of Communications and the Australian Research Council through the ICT Centre of Excellence Program.}\\
{\tt\small \{dylan.campbell,lars.petersson\}@anu.edu.au}
}

\maketitle

\begin{abstract}
\vspace{-6pt}

Gaussian mixture alignment is a family of approaches that are frequently used for robustly solving the point-set registration problem. However, since they use local optimisation, they are susceptible to local minima and can only guarantee local optimality. Consequently, their accuracy is strongly dependent on the quality of the initialisation. This paper presents the first globally-optimal solution to the 3D rigid Gaussian mixture alignment problem under the $L_2$ distance between mixtures. The algorithm, named GOGMA, employs a branch-and-bound approach to search the space of 3D rigid motions $SE(3)$, guaranteeing global optimality regardless of the initialisation. The geometry of $SE(3)$ was used to find novel upper and lower bounds for the objective function and local optimisation was integrated into the scheme to accelerate convergence without voiding the optimality guarantee. The evaluation empirically supported the optimality proof and showed that the method performed much more robustly on two challenging datasets than an existing globally-optimal registration solution.

\end{abstract}

\vspace{-6pt}
\vspace{-6pt}
\section{Introduction}
\label{sec:introduction}

Gaussian Mixture Alignment (GMA), the problem of finding the transformation that best aligns one Gaussian mixture with another, has been investigated extensively in computer vision and robotics. It has a natural application to point-set registration, which endeavours to solve the same problem as GMA for discrete point-sets in $\mathbb{R}^n$. Indeed, the Iterative Closest Point (ICP) algorithm~\cite{besl1992method, zhang1994iterative} and several other local registration algorithms~\cite{chui2000new, chui2000feature, tsin2004correlation, myronenko2010point}
can be interpreted as special cases of GMA~\cite{jian2011robust}.
Applications include merging multiple partial scans into a complete model~\cite{blais1995registering,huber2003fully}; using registration results as fitness scores for object recognition~\cite{johnson1999using,belongie2002shape}; registering a view into a global coordinate system for sensor localisation~\cite{nuchter20076d,pomerleau2013comparing}; and finding relative poses between sensors~\cite{yang2013single,geiger2012automatic}.

The dominant solution for 2D and 3D rigid registration is the ICP algorithm~\cite{besl1992method, zhang1994iterative} and variants, due to its conceptual simplicity, ease of use and good performance. However, ICP is limited by its assumption that closest point pairs should correspond, which fails when the point-sets are not coarsely aligned, or the moving `model' point-set is not a proper subset of the static `scene' point-set. The former means that, without a good initialisation, ICP is very susceptible to local minima, producing erroneous estimates without a reliable means of detecting failure. The latter regularly occurs, since differing sensor viewpoints and dynamic objects lead to occlusion and partial-overlap.

Gaussian mixture alignment~\cite{chui2000feature, tsin2004correlation, jian2011robust, campbell2015adaptive} was introduced to address these shortcomings. By aligning point-sets without establishing explicit point correspondences, GMA is less sensitive to missing correspondences from partial overlap or occlusion and mitigate the problem of local minima by widening the basin of convergence. Robust objective functions can also be applied, such as the $L_2$ distance between mixtures~\cite{jian2011robust, campbell2015adaptive}. However, GMA still requires a good initialisation and cannot guarantee global optimality.

Go-ICP~\cite{yang2013goicp,yang2016goicp} was the first globally-optimal algorithm for the 3D rigid registration problem defined by ICP. Specifically, it used a branch-and-bound approach to find the global minimum of the ICP error metric, the $L_2$ norm of closest-point residuals. Despite solving the problem of local minima, Go-ICP inherits the non-robust ICP cost function that is susceptible to occlusion and partial overlap. Yang \etal~\cite{yang2013goicp} proposed a trimming strategy to handle outlier correspondences. However, this increased the runtime and required an additional trimming parameter to be set.

This work is the first to propose a globally-optimal solution to the 3D Gaussian mixture alignment problem under Euclidean (rigid) transformations. It inherits the robust $L_2$ density distance objective function of $L_2$ GMA while avoiding the problem of local minima. The method, named GOGMA, employs the branch-and-bound algorithm to guarantee global optimality regardless of initialisation, using a parametrisation of $SE(3)$ space that facilitates branching. The pivotal contribution is the derivation of the objective function bounds using the geometry of $SE(3)$. In addition, local GMA optimisation is applied whenever the algorithm finds a better transformation, to accelerate convergence without voiding the optimality guarantee.

\section{Related Work}
\label{sec:related_work}

The large quantity of work published on ICP, its variants and other registration techniques precludes a comprehensive list. The reader is directed to the surveys on ICP variants~\cite{rusinkiewicz2001efficient,pomerleau2013comparing} and recent 3D point-set and mesh registration techniques~\cite{tam2013registration} for additional background.
To improve the robustness of ICP to occlusion and partial overlap, approaches included trimming~\cite{chetverikov2005robust} and outlier rejection~\cite{zhang1994iterative,granger2002multi}.
To enlarge ICP's basin of convergence, approaches included LM-ICP~\cite{fitzgibbon2003robust}, which used the Levenberg-Marquardt algorithm~\cite{more1978levenberg} and a distance transform to optimise the ICP error, without establishing explicit point correspondences.
The family of Gaussian mixture alignment approaches also sought to improve robustness to poor initialisations, noise and outliers. Notable GMA-related algorithms for rigid and non-rigid registration include Robust Point Matching~\cite{chui2003new}, Coherent Point Drift~\cite{myronenko2010point}, Kernel Correlation~\cite{tsin2004correlation} and GMMReg~\cite{jian2011robust}. GMMReg~\cite{jian2011robust} defines an equally-weighted Gaussian at every point in the set with identical and isotropic covariances and minimises the robust $L_2$ distance between densities.
The Normal Distributions Transform (NDT) algorithm~\cite{magnusson2007scan, stoyanov2012fast} defines Gaussians for every cell in a grid and estimates full data-driven covariances.
SVR~\cite{campbell2015adaptive} uses an SVM to construct a Gaussian mixture with non-uniform weights that adapts to the structure of the point-set and is robust to occlusion and partial overlap. While more robust than ICP, these methods all employ local optimisation, which is dependent on the initial pose.

There are many heuristic or stochastic methods for global alignment that are not guaranteed to converge. One class utilises stochastic optimisation techniques, such as particle filtering~\cite{sandhu2010point}, genetic algorithms~\cite{silva2005precision,robertson2002parallel} and simulated annealing \cite{blais1995registering, papazov2011stochastic}. Another class is feature-based alignment, which exploits the transformation invariance of a local descriptor to build sparse feature correspondences, such as fast point feature histograms~\cite{rusu2009fast}. The transformation can be found from the correspondences using random sampling~\cite{rusu2009fast}, greedy algorithms~\cite{johnson1999using}, Hough transforms~\cite{woodford2014demisting} or branch-and-bound~\cite{gelfand2005robust,bazin2012globally}.
\textsc{Super 4PCS}~\cite{mellado2014super} is a recent example of a method that uses random sampling without features. It is a four-points congruent sets method that exploits a clever data structure to achieve linear-time performance, extending the original 4PCS algorithm~\cite{aiger20084}.

In contrast, globally-optimal techniques avoid local minima by searching the entire transformation space, often using the branch-and-bound paradigm. Existing 3D methods~\cite{li20073d, olsson2009branch, bustos2014fast, yang2013goicp} are often very slow or make restrictive assumptions about the point-sets, correspondences or transformations. For example, Li and Hartley~\cite{li20073d} minimised a Lipschitzized $L_2$ error function using branch-and-bound, but assumed that the point-sets were the same size and the transformation was pure rotation. Olsson \etal~\cite{olsson2009branch} found optimal solutions to point-to-point/line/plane registration using branch-and-bound and bilinear relaxation of rotation quaternions, but assumed correspondences were known. Recently, Bustos \etal~\cite{bustos2014fast} achieved efficient run-times using stereographic projection techniques for optimal 3D alignment, but assumed that translation was known.
Finally, Yang \etal~\cite{yang2013goicp,yang2016goicp} proposed the Go-ICP algorithm, which finds the optimal solution to the closest point $L_2$ error between point-sets and is accelerated by using local ICP as a sub-routine.
However, it is sensitive to occlusion and partial overlap, due to its non-robust cost function. The proposed trimming strategy goes some way to alleviating this, but increases the runtime, requires an estimate of the overlap percentage and may lead to ambiguity in the solution.

The rest of the paper is organised as follows: we outline Gaussian mixture alignment in Section~\ref{sec:gma}, we develop a parametrisation of the domain of 3D motions, a branching strategy and a derivation of the bounds in Section~\ref{sec:bb}, we propose an algorithm for globally-optimal GMA in Section~\ref{sec:gogma} and we evaluate the its performance in Section~\ref{sec:experimental_results}.

\section{Gaussian Mixture Alignment}
\label{sec:gma}

The alignment of Gaussian Mixture Models (GMMs) to solve the point-set registration task is a well-studied problem \cite{chui2000feature,tsin2004correlation,magnusson2007scan,jian2011robust,campbell2015adaptive}. GMMs can be constructed from point-set data using Kernel Density Estimation (KDE) \cite{jian2011robust,detry2009probabilistic,comaniciu2003algorithm,xiong2013study}, Expectation Maximisation (EM) \cite{dempster1977maximum,deselaers2010object} or a Support Vector Machine (SVM) \cite{campbell2015adaptive}. Once the point-sets are in GMM form, the registration problem can be posed as minimising a discrepancy measure between GMMs. We use the $L_2$ distance formulation of Jian and Vemuri \cite{jian2011robust}, which can be expressed in closed-form and efficiently implemented. The $L_2 E$ estimator minimises the $L_2$ distance between densities and is inherently robust to outliers \cite{scott2001parametric}.

Let $\mathcal{G}_{\mathcal{X}} \!=\! \{ \bx_i , \sigma_{i\mathcal{X}}^2 , \phi_i^{\mathcal{X}} \}_{i = 1}^{m}$ and $\mathcal{G}_{\mathcal{Y}} \!=\! \{ \by_j , \sigma_{j\mathcal{Y}}^2 , \phi_j^{\mathcal{Y}} \}_{j = 1}^{n}$ be GMMs constructed from point-sets $\mathcal{X}$ and $\mathcal{Y}$, with means $\bx_i$ and $\by_j$, variances $\sigma_{i\mathcal{X}}^2$ and $\sigma_{j\mathcal{Y}}^2$, mixture weights $\phi_i^{\mathcal{X}}$ and $\phi_j^{\mathcal{Y}}$ and number of components $m$ and $n$ respectively. Also let $T(\mathcal{G},\bR,\bt)$ be the function that rigidly transforms $\mathcal{G}$ with rotation $\bR \in SO(3)$ and translation $\bt \in \mathbb{R}^3$. The $L_2$ distance $D$ between transformed $\mathcal{G}_{\mathcal{X}}$ and $\mathcal{G}_{\mathcal{Y}}$ is given by
\begin{equation}
\label{eqn:l2_distance}
D(\bR,\bt) = \int_{\mathbb{R}^{{}^{3}}} \left( p\left( \bp \middle| T(\mathcal{G}_{\mathcal{X}}, \bR,\bt) \right) - p\left( \bp \middle| \mathcal{G}_{\mathcal{Y}} \right) \right)^{2}\,\mathrm{d}\bp
\end{equation}
where $p\left( \bp \middle| \mathcal{G} \right)$ is the probability of observing a point $\bp$ given a mixture model $\mathcal{G} = \{ \boldsymbol{\mu}_i , \sigma_{i}^{2} , \phi_i \}_{i = 1}^{\ell}$, that is
\begin{equation}
\label{eqn:gmm_probability}
p\left( \bp \middle| \mathcal{G} \right) = \sum_{i = 1}^{\ell} \phi_{i} \mathcal{N} \left( \bp \middle| \boldsymbol{\mu}_{i} , \sigma_{i}^{2} \right)
\end{equation}
where $\mathcal{N} \left( \bp \middle| \boldsymbol{\mu} , \sigma^{2} \right)$ is the probability of the Gaussian at $\bp$.
Expanding (\ref{eqn:l2_distance}), $\left[ p\left( \bp \middle| T(\mathcal{G}_{\mathcal{X}}, \bR,\bt) \right) \right]^2$ is invariant under rigid transformations and $\left[ p\left( \bp \middle| \mathcal{G}_{\mathcal{Y}} \right) \right]^2$ is independent of $(\bR,\bt)$. The integral of the $-2 p\left( \bp \middle| T(\mathcal{G}_{\mathcal{X}}, \bR,\bt) \right) p\left( \bp \middle| \mathcal{G}_{\mathcal{Y}} \right)$ term has a closed form, derived by applying the identity
\begin{multline}
\label{eqn:gmm_identity}
\int_{\mathbb{R}^{{}^{3}}} \mathcal{N} \left( \bp \middle| \boldsymbol{\mu}_{1} , \sigma_{1}^{2} \right) \mathcal{N} \left( \bp \middle| \boldsymbol{\mu}_{2} , \sigma_{2}^{2} \right) \,\mathrm{d}\bp\\
= \mathcal{N} \left( \mathbf{0} \middle| \boldsymbol{\mu}_{1} - \boldsymbol{\mu}_{2} , \sigma_{1}^{2} + \sigma_{2}^{2} \right).
\end{multline}
Therefore, by substituting in (\ref{eqn:gmm_probability}) and (\ref{eqn:gmm_identity}), the objective function for Gaussian Mixture Alignment (GMA) is given by
\begin{align}
f\left(\bR, \bt \right) &= - \int_{\mathbb{R}^{{}^{3}}} p\left( \bp \middle| T(\mathcal{G}_{\mathcal{X}}, \bR,\bt) \right) p\left( \bp \middle| \mathcal{G}_{\mathcal{Y}} \right)\,\mathrm{d}\bp\nonumber\\
&= - \!\! \sum_{i = 1}^{m} \sum_{j = 1}^{n} \phi_{i}^{\mathcal{X}} \phi_{j}^{\mathcal{Y}} \mathcal{N} \! \left( \mathbf{0} \middle| \bR \bx_i \!+ \bt \!- \by_j, \sigma_{i\mathcal{X}}^{2} \!+ \sigma_{j\mathcal{Y}}^{2} \right)\nonumber\\
&= - \!\! \sum_{i = 1}^{m} \sum_{j = 1}^{n} \frac{\phi_{i}^{\mathcal{X}} \phi_{j}^{\mathcal{Y}}}{Z} \exp \left[ -\frac{\left[e_{ij}\left(\bR, \bt\right) \right]^2}{2 \left[ \sigma_{i\mathcal{X}}^2 + \sigma_{j\mathcal{Y}}^2 \right]} \right]
\label{eqn:objective_function}
\end{align}
where $Z$ is the normalisation factor and $e_{ij}\left(\bR, \bt\right)$ is the pairwise residual error. For local GMA, we use the quasi-Newton L-BFGS-B algorithm \cite{byrd1995limited} to minimise (\ref{eqn:objective_function}).

\section{Branch-and-Bound}
\label{sec:bb}

Branch-and-Bound (BB) \cite{land1960automatic} is a global optimisation technique that is frequently used to solve non-convex and NP-hard problems~\cite{lawler1966branch}. To apply the technique to 3D GMA, a suitable means of parametrising and branching (partitioning) the domain of 3D rigid transformations $SE(3)$ must be found, as well as an efficient way to find the upper and lower bounds of each branch. The latter is critical for the time and memory efficiency of the algorithm, since tight bounds that are quick to evaluate reduce the search space by allowing suboptimal branches to be pruned.

\subsection{Parametrising and Branching the Domain}
\label{sec:bb_parametrisation}

For the globally-optimal alignment problem, the objective function (\ref{eqn:objective_function}) must be minimised over the domain of 3D motions, that is, the group $SE(3) = SO(3) \times \mathbb{R}^3$. One minimal parametrisation of $SO(3)$ is the angle-axis representation, where a rotation is parametrised as a 3-vector $\br$ with a rotation angle $\|\br\|$ and rotation axis $\br/\|\br\|$. We use the notation $\bR_{\br} \in SO(3)$ to denote the rotation matrix obtained from the matrix exponential map of the skew-symmetric matrix $[\br]_{\times}$ induced by $\br$. The Rodrigues' rotation formula can be used to efficiently calculate this mapping. The space of all 3D rotations can therefore be represented as a solid ball of radius $\pi$ in $\mathbb{R}^3$. The mapping is one-to-one on the interior of the $\pi$-ball and two-to-one on the surface. For ease of manipulation, we use the 3D cube circumscribing the $\pi$-ball as the rotation domain~\cite{li20073d}.
The translation domain is chosen as the bounded cube $[-\tau,\tau]^3$. If the GMMs were constructed from point-sets scaled to fit within $[-0.5, 0.5]^3$, choosing $\tau = 1$ would ensure that the domain covered every feasible translation. In practice, a smaller $\tau$ is acceptable, if a minimum detectable bounding box overlap is acceptable.

Together, these domains form a 6D hypercube, shown separately in Figure~\ref{fig:domain}. During BB, the hypercube is branched using a hyperoctree data structure. The uncertainty region induced by a hypercube on a point $\bx$ is shown for rotation and translation separately in Figure~\ref{fig:uncertainty_region}. The transformed point may lie anywhere in the uncertainty region, which is the Minkowski sum of a spherical patch and a cube for rotation and translation dimensions respectively.

\begin{figure}[!t]
\centering
\begin{subfigure}[]{0.495\columnwidth}
\centering
\def\svgwidth{0.65\columnwidth}
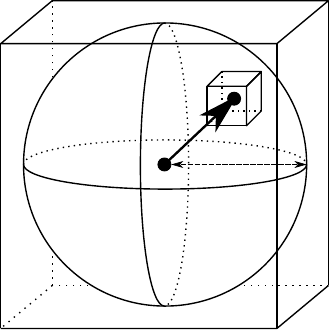
\caption{Rotation Domain}
\label{fig:domain_rotation}
\end{subfigure}
\hfill
\begin{subfigure}[]{0.495\columnwidth}
\centering
\def\svgwidth{0.65\columnwidth}
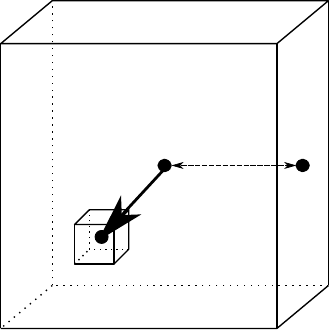
\caption{Translation Domain}
\label{fig:domain_translation}
\end{subfigure}
\caption{Parametrisation of $SE(3)$. (\subref{fig:domain_rotation})~The rotation space $SO(3)$ is parametrised by angle-axis 3-vectors within a solid radius-$\pi$ ball. (\subref{fig:domain_translation})~The translation space $\mathbb{R}^3$ is parametrised by 3-vectors within a cube of half side-length $\tau$. The joint domain is branched using a hyperoctree data structure, with a sub-hypercube depicted as two sub-cubes in the rotation and translation dimensions.}
\label{fig:domain}
\end{figure}

\begin{figure}[!t]
\centering
\begin{subfigure}[]{0.495\columnwidth}
\def\svgwidth{\columnwidth}
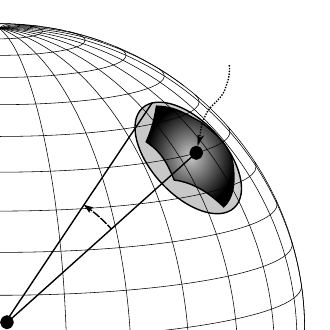
\caption{Rotation Uncertainty Region}
\label{fig:uncertainty_region_rotation}
\end{subfigure}
\hfill
\begin{subfigure}[]{0.495\columnwidth}
\def\svgwidth{\columnwidth}
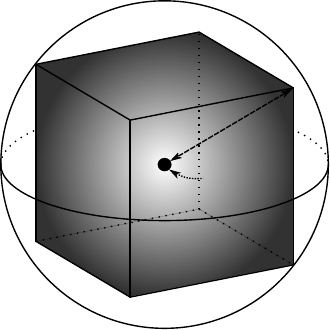
\caption{Translation Uncertainty Region}
\label{fig:uncertainty_region_translation}
\end{subfigure}
\caption{Uncertainty region induced by hypercube $C = C_r \times C_t$. (\subref{fig:uncertainty_region_rotation})~Rotation uncertainty region for $C_r$ with centre $\bR_{\br_{0}}\bx$. The optimal rotation of $\bx$ may be anywhere within the heavily-shaded umbrella-shaped uncertainty region, which is entirely contained by the lightly-shaded spherical cap defined by $\bR_{\br_{0}}\bx$ and $\beta$. (\subref{fig:uncertainty_region_translation})~Translation uncertainty region for $C_t$ with centre $\bx + \bt_{0}$. The optimal translation of $\bx$ may be anywhere within the cube, which is entirely contained by the circumscribed sphere with radius $\rho$.
}
\label{fig:uncertainty_region}
\end{figure}

\subsection{Bounding the Branches}
\label{sec:bb_bounds}

The success of a BB algorithm is predicated on the quality of its bounds. For rigid 3D Gaussian mixture alignment, the GMA objective function (\ref{eqn:objective_function}) within a transformation domain $C_r\times C_t$ needs to be bounded.
To simplify the bounds, the GMMs are assumed to have isotropic covariances.
The translation component $\bt \in C_t$ can be bounded by a single value $\rho$ by observing that the translation cube can be inscribed in a sphere (Figure~\ref{fig:uncertainty_region_translation}), as in \cite{yang2013goicp}.
\begin{lemma}
\label{lm:translation_uncertainty}
(Translation sphere radius) Given a 3D point $\bx$ and a translation cube $C_t$ of half side-length $\delta_t$ centred at $\bt_0$, then $\forall \bt \in C_t$,
\begin{equation}
\label{eqn:translation_uncertainty}
\|(\bx + \bt) - (\bx + \bt_0)\| \leqslant \sqrt{3} \delta_t \doteq \rho.
\end{equation}
\end{lemma}
\begin{proof}
Inequality (\ref{eqn:translation_uncertainty}) can be derived by observing that $\|(\bx + \bt) - (\bx + \bt_0)\| = \|\bt - \bt_0\|$ and $\max \|\bt - \bt_0\| = \sqrt{3} \delta_t$ (the half space diagonal) for $\bt \in C_t$.
\end{proof}

To bound the rotation component $\br \in C_r$, Lemmas 1 and 2 from \cite{hartley2009global} are used. For reference, the relevant parts are merged into Lemma~\ref{lm:rotation_inequality}, as in \cite{yang2016goicp}. The lemma indicates that the angle between two rotated vectors is less than or equal to the Euclidean distance between their rotations' angle-axis representations in $\mathbb{R}^3$.
\begin{lemma}
\label{lm:rotation_inequality}
For an arbitrary vector $\bx$ and two rotations, represented as $\bR_{\br_{1}}$ and $\bR_{\br_{2}}$ in matrix form and $\br_{1}$ and $\br_{2}$ in angle-axis form,
\begin{equation}
\label{eqn:rotation_inequality}
\angle(\bR_{\br_{1}}\bx, \, \bR_{\br_{2}} \bx) \leqslant \|\br_{1} - \br_{2}\|.
\end{equation}
\end{lemma}

From this, the maximum angle $\beta$ between a vector $\bx$ rotated by $\br_0$ and $\bx$ rotated by $\br \in C_r$ can be found as follows.
\begin{lemma}
\label{lm:maximum_aperture_angle}
(Maximum aperture angle) Given a 3D point $\bx$ and a rotation cube $C_r$ of half side-length $\delta_r$ centred at $\br_0$, then $\forall \br \in C_r$,
\begin{equation}
\label{eqn:maximum_aperture_angle}
\angle(\bR_\br \bx, \bR_{\br_0} \bx) \leqslant \min(\sqrt{3}\delta_r,\pi) \doteq \beta.
\end{equation}
\end{lemma}
\begin{proof}
Inequality (\ref{eqn:maximum_aperture_angle}) can be derived as follows:
\begin{align}
\angle(\bR_\br \bx, \bR_{\br_0} \bx)
&\leqslant \min({\|\br-\br_0\|},{\pi})
\label{eqn:maximum_aperture_angle_1}\\
&\leqslant \min({\sqrt{3}\delta_r},{\pi})
\label{eqn:maximum_aperture_angle_2}
\end{align}
where (\ref{eqn:maximum_aperture_angle_1}) follows from Lemma~\ref{lm:rotation_inequality} and (\ref{eqn:maximum_aperture_angle_2}) follows from $\max \|\br - \br_0\| = \sqrt{3} \delta_r$ (the half space diagonal of the rotation cube) for $\br \in C_r$.
\end{proof}

As a first step towards bounding the GMA objective function, bounds for the pairwise residual error ${e}_{ij}(\bR_\br,\bt)$ need to be found. This ${e}_{ij}(\bR_\br,\bt)$ represents the minimal $L_2$ distance between the Gaussian means $(\bR_\br\bx_i + \bt)$ and $\by_j$ for $(\br,\bt) \in (C_r\times C_t)$. For convenience, let $\by_{j}^{\prime} = (\by_j - \bt_0)$.

\begin{theorem}
\label{thm:bounds_pairwise}
(Bounds of the pairwise residual error)
For the 3D transformation domain $C_r\times C_t$ centred at $(\br_0,\bt_0)$ with translation sphere radius $\rho$, the upper bound $\bar{e}_{ij}$ and the lower bound $\ubar{e}_{ij}$ of the optimal pairwise residual error $e_{ij}(\bR_\br,\bt)$ for $\bx_i$ and $\by_j$ can be chosen as
\begin{align}
\bar{e}_{ij}(\bR_\br,\bt) &= e_{ij}(\bR_{\br_0},\bt_0)
\label{eqn:pairwise_upper_bound}\\
\ubar{e}_{ij}(\bR_\br,\bt) &= \max \Big( \big\lvert \| \bx_{i}\| - \| \by_{j} - \bt_{0} \| \big\rvert - \rho, \, 0 \Big).
\label{eqn:pairwise_lower_bound}
\end{align}
\end{theorem}
\begin{proof}
\label{prf:bounds_pairwise}
The validity of the upper bound $\bar{e}_{ij}$ follows from the error $e_{ij}$ at a specific point within the domain being larger than the minimal error within the domain, that is
\begin{equation}
\label{eqn:pairwise_upper_bound_proof}
e_{ij}(\bR_{\br_0},\bt_0) \geqslant \min_{\forall (\br,\bt)\in (C_r\times C_t)} e_{ij}(\bR_{\br},\bt).
\end{equation}
The validity of the lower bound $\ubar{e}_{ij}$ is proved as follows. Observe that $\forall (\br,\bt)\in (C_r\times C_t)$,
\begin{align}
e_{ij}(\bR_{\br},\bt) &= \| \bR_{\br} \bx_i + \bt - \by_j \|
\label{eqn:pairwise_lower_bound_proof_1}\\
&= \| \bR_{\br} \bx_i - (\by_j - \bt_0) - (\bt_0 - \bt) \|
\label{eqn:pairwise_lower_bound_proof_2}\\
&\geqslant \big\lvert \| \bR_{\br} \bx_i - \by_j^{\prime} \| - \| (\bt_0 - \bt) \| \big\rvert
\label{eqn:pairwise_lower_bound_proof_3}\\
&\geqslant \max \big( \| \bR_{\br} \bx_i - \by_j^{\prime} \| - \| (\bt_0 - \bt) \| , \, 0 \big)
\label{eqn:pairwise_lower_bound_proof_4}\\
&\geqslant \max \big( \| \bR_{\br} \bx_i - \by_j^{\prime} \| - \rho , \, 0 \big)
\label{eqn:pairwise_lower_bound_proof_5}\\
&\geqslant \max \big( \big\lvert \| \bR_{\br} \bx_i \| - \| \by_j^{\prime}\| \big\rvert - \rho , \, 0 \big)
\label{eqn:pairwise_lower_bound_proof_6}\\
&= \max \big( \big\lvert \| \bx_i \| - \| \by_j - \bt_0 \| \big\rvert - \rho , \, 0 \big)
\label{eqn:pairwise_lower_bound_proof_7}
\end{align}
where (\ref{eqn:pairwise_lower_bound_proof_2}) adds and subtracts $\bt_0$, (\ref{eqn:pairwise_lower_bound_proof_3}) follows from the reverse triangle inequality $\|\bx - \by\| \geqslant \lvert \|\bx\| - \|\by\| \rvert$, (\ref{eqn:pairwise_lower_bound_proof_4}) from the absolute value of a quantity being positive, (\ref{eqn:pairwise_lower_bound_proof_5}) follows from (\ref{eqn:translation_uncertainty}), and (\ref{eqn:pairwise_lower_bound_proof_6}) from the reverse triangle inequality.
\end{proof}

The geometric intuition of Theorem~\ref{thm:bounds_pairwise} is that all valid points $(\bR_{\br} \bx_i + \bt)$ lie within $\rho$ of the rotation sphere centred at $\bt_0$ with radius $\|\bx_i\|$.
However, the gap between this `sphere distance' pairwise lower bound and the pairwise upper bound does not converge to zero as the sub-cube sizes decrease, since the lower bound is independent of the rotation sub-cube size $\delta_r$.
We can find a converging lower bound by observing that all valid points, neglecting translation uncertainty, lie on a spherical cap. Letting $\bx_{i}^{0} = \bR_{\br_{0}}\bx_i$, the spherical cap is defined by the sphere of radius $\|\bx_i\|$ centred at $\bt_0$ with the constraint that $\angle \left(\bx_{i}^{0}, \bx\right) \leqslant \beta$ for all points $(\bx + \bt_0)$ on the cap.
Now let $\bx_{i}^{\prime} = \bR\bx_i$ be an arbitrary point on the origin-centred spherical cap and let $\bx_{ij}^{*} = \bR\bx_i$ be the point coplanar with $\bx_{i}^{0}$ and $\by_{j}^{\prime}$ on the edge of the spherical cap, such that $\angle \left( \bx_{i}^{0}, \bx_{ij}^{*} \right) = \beta$.

\begin{theorem}
\label{thm:pairwise_lower_bound_improved}
(Tight lower bound of the pairwise residual error)
For the 3D transformation domain $C_r\times C_t$ centred at $(\br_0,\bt_0)$ with translation sphere radius $\rho$, the lower bound $\ubar{e}_{ij}$ of the optimal pairwise residual error $e_{ij}(\bR_\br,\bt)$ for $\bx_i$ and $\by_j$ can be chosen as
\begin{equation}
\label{eqn:pairwise_lower_bound_improved}
\ubar{e}_{ij}(\bR_\br,\bt) \!=\!
\begin{dcases}
\!\max\! \Big[ \big\lvert \| \bx_{i}\| \!-\! \| \by_{j}^{\prime} \| \big\rvert \!-\! \rho, 0 \Big] \!\!\!\!& \text{for } \alpha \!\leqslant\! \beta\\
\!\max\! \Big[ \| \bx_{ij}^* \!-\! \by_j^{\prime} \| \!-\! \rho, 0 \Big] \!\!\!\!& \text{for } \alpha \!>\! \beta
\end{dcases}
\end{equation}
where $\alpha$ and $\beta$ are shown in Figure~\ref{fig:bounds} and are given by
\begin{align}
\alpha &= \angle \left( \bx_i^0, \by_j^{\prime} \right)
\label{eqn:alpha}\\
\beta &= \angle \left( \bx_i^0, \bx_{ij}^* \right) = \min(\sqrt{3}\delta_r,\pi).
\label{eqn:beta}
\end{align}
\end{theorem}
\begin{proof}
\label{prf:pairwise_lower_bound_improved}
Observe that $\forall (\br,\bt)\in (C_r\times C_t)$,
\begin{align}
e_{ij}(\bR_{\br},\bt)
&\geqslant \max \big( \| \bR_{\br} \bx_i - \by_j^{\prime} \| - \rho , \, 0 \big)
\label{eqn:pairwise_lower_bound_improved_proof_1}\\
&\geqslant \max \big( \min \| \bR_{\br} \bx_i - \by_j^{\prime} \| - \rho , \, 0 \big)
\label{eqn:pairwise_lower_bound_improved_proof_2}\\
&\geqslant \max \big( \min \| \bx_i^{\prime} - \by_j^{\prime} \| - \rho , \, 0 \big)
\label{eqn:pairwise_lower_bound_improved_proof_3}\\
&= \!
\begin{dcases}
\!\max\! \Big[ \big\lvert \| \bx_{i}\| \!-\! \| \by_{j}^{\prime} \| \big\rvert \!-\! \rho, 0 \Big] \!\!\!\!\!& \text{for } \alpha \!\leqslant\! \beta\\
\!\max\! \Big[ \| \bx_{ij}^* \!-\! \by_j^{\prime} \| \!-\! \rho, 0 \Big] \!\!\!\!\!& \text{for } \alpha \!>\! \beta
\end{dcases}
\label{eqn:pairwise_lower_bound_improved_proof_4}
\end{align}
where (\ref{eqn:pairwise_lower_bound_improved_proof_1}) is transferred from (\ref{eqn:pairwise_lower_bound_proof_5}), (\ref{eqn:pairwise_lower_bound_improved_proof_3}) states that the minimum distance to a constrained point on the spherical cap is greater than or equal to the minimum distance to an unconstrained point on the cap and (\ref{eqn:pairwise_lower_bound_improved_proof_4}) is from Theorem~\ref{thm:spherical_cap_distance}.
\end{proof}

\begin{theorem}
\label{thm:spherical_cap_distance}
(Spherical cap distance)
For the spherical cap defined by the vector $\bx_i^{\prime} + \bt_0$ constrained by $\angle (\bx_i^{\prime}, \bx_i^{0}) \leqslant \beta$, the minimum distance from a point $\by_j$ to the spherical cap is given by
\begin{equation}
\label{eqn:spherical_cap_distance}
\min \! \|\bx_i^{\prime} + \bt_0 - \by_j \| \!=\!
\begin{dcases}
\!\big\lvert \| \bx_{i}\| - \| \by_{j} - \bt_{0} \| \big\rvert \!\!\!\!& \text{for } \alpha \!\leqslant\! \beta\\
\!\| \bx_{ij}^* - (\by_j - \bt_0) \| \!\!\!\!& \text{for } \alpha \!>\! \beta
\end{dcases}
\end{equation}
where $\alpha$ and $\beta$ are as per Theorem~\ref{thm:pairwise_lower_bound_improved}, shown in Figure~\ref{fig:bounds}.
\end{theorem}
\begin{proof}
Dropping the subscripts and translating everything by $(-\bt_0)$, an arbitrary point $\bx^{\prime}$ on the spherical cap can be expressed as the rotation of the point $\bx^{0}$ about the sphere centre towards $\by^{\prime}$ by an angle $\gamma \in [0,\beta]$, followed by a rotation of this intermediate vector about the axis $\bx^{0}$ by $\theta$. By two applications of the Rodrigues' rotation formula,
\begin{equation}
\label{eqn:arbitrary_vector}
\bx^{\prime} = \left[ \shortcos_{\gamma} - \frac{\shortcos_{\alpha} \shortsin_{\gamma} \shortcos_{\theta}}{\shortsin_{\alpha}} \right] \bx^{0} + \frac{\shortsin_{\gamma}}{\shortsin_{\alpha}} \frac{\bx^0 \times \by^{\prime}}{\|\by^{\prime}\|} + \frac{\shortsin_{\gamma} \shortcos_{\theta} \|\bx\|}{\shortsin_{\alpha}\|\by^{\prime}\|} \by^{\prime}
\end{equation}
where $\shortsin_{A} = \sin A$ and $\shortcos_{A} = \cos A$. By substitution, the squared distance between the point $\by^{\prime}$ and an arbitrary point on the spherical cap is given by
\begin{equation}
\label{eqn:spherical_cap_distance_proof}
\|\bx^{\prime} \!- \by^{\prime} \|^2 \!=\! \|\bx\|^2 \!+ \|\by^{\prime}\|^2 \!- 2 \left[ \shortcos_{\alpha} \! \shortcos_{\gamma} \!+ \shortsin_{\alpha} \! \shortsin_{\gamma} \! \shortcos_{\theta} \right] \!\|\bx\|\|\by^{\prime}\|
\end{equation}
and is minimised when $\theta = 0$. Therefore,
\begin{equation}
\label{eqn:spherical_cap_distance_proof_minimum}
\min \|\bx^{\prime} \!- \by^{\prime} \|^2 \!= \|\bx\|^2 \!+ \|\by^{\prime}\|^2 \!- 2 \cos[\alpha - \gamma] \|\bx\|\|\by^{\prime}\|.
\end{equation}
When $\alpha \leqslant \beta$ (Case 1), (\ref{eqn:spherical_cap_distance_proof_minimum}) is minimised when $\gamma = \alpha$:
\begin{align}
\phantom{\therefore} \min \|\bx^{\prime} - \by^{\prime} \|^2 &= \left(\|\bx\| - \|\by^{\prime}\| \right)^2
\label{eqn:spherical_cap_distance_proof_minimum_case1_1}\\
\therefore \min \|\bx^{\prime} - \by^{\prime} \|\phantom{^2} &= \big\lvert\|\bx\| - \|\by^{\prime}\| \big\rvert \text{ for } \alpha \leqslant \beta.
\label{eqn:spherical_cap_distance_proof_minimum_case1_2}
\end{align}
When $\alpha > \beta$ (Case 2), (\ref{eqn:spherical_cap_distance_proof_minimum}) is minimised when $\gamma = \beta$:
\begin{align}
\phantom{\therefore} \min \|\bx^{\prime} - \by^{\prime} \|^2 &= \|\bx^{*} - \by^{\prime}\|^2
\label{eqn:spherical_cap_distance_proof_minimum_case2_1}\\
\therefore \min \|\bx^{\prime} - \by^{\prime} \|\phantom{^2} &= \|\bx^{*} - \by^{\prime}\| \text{ for } \alpha > \beta.
\label{eqn:spherical_cap_distance_proof_minimum_case2_2}
\end{align}
The full proof is left for the appendix.
\end{proof}

The geometric intuition, shown in Figure~\ref{fig:bounds}, for Theorems~\ref{thm:pairwise_lower_bound_improved} and \ref{thm:spherical_cap_distance} is that the minimum distance to the spherical cap is equal to (i) the radial distance to the sphere if $\by$ lies within the rotation cone
$(\alpha \leqslant \beta)$
or (ii) the distance to the edge of the cap.
This `spherical cap distance' pairwise lower bound is a tighter bound than (\ref{eqn:pairwise_lower_bound}) and the gap between it and the upper bound converges to zero. This can be shown by observing that the size of the spherical cap diminishes with the size of the rotation sub-cube~(\ref{eqn:maximum_aperture_angle}) and likewise for the translation sphere and translation sub-cube~(\ref{eqn:translation_uncertainty}).
It is also a tighter bound than that in \cite{yang2013goicp}, which uses the distance to a sphere enclosing the spherical cap.

\begin{figure}[!t]
\centering
\begin{subfigure}[]{0.82\columnwidth}
\def\svgwidth{\columnwidth}
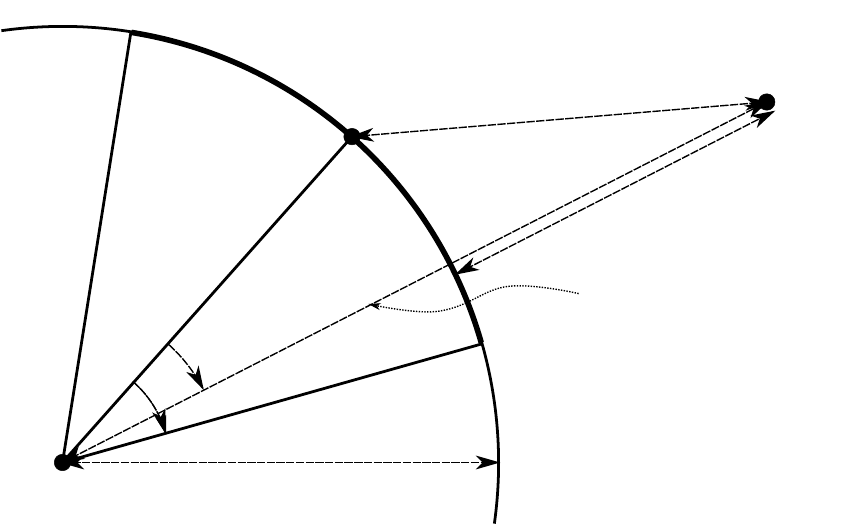
\caption{Case 1: $\by_j$ is within the rotation cone $(\alpha \leqslant \beta)$.}
\label{fig:bounds_case1}
\end{subfigure}

\begin{subfigure}[]{0.82\columnwidth}
\def\svgwidth{\columnwidth}
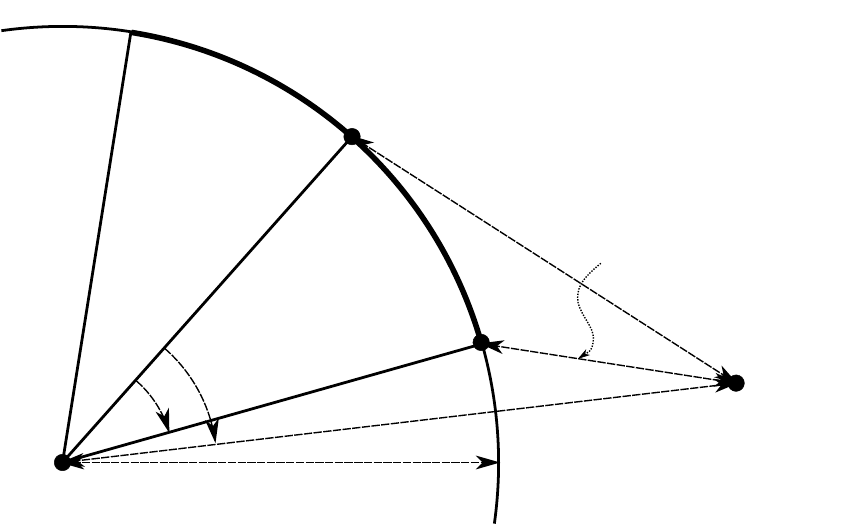
\caption{Case 2: $\by_j$ is outside the rotation cone $(\alpha > \beta)$.}
\label{fig:bounds_case2}
\end{subfigure}
\caption{Upper and lower bound of the pairwise residual error, neglecting translation. A 2D cross-section of the plane defined by points $\{\mathbf{R}_{\mathbf{r}_{0}} \mathbf{x}_i + \mathbf{t}_0, \mathbf{y}_j, \mathbf{t}_0\}$ is shown.
}
\label{fig:bounds}
\end{figure}

The bounds of the objective function are found by summing the kernelised upper and lower bounds of the pairwise residual errors in (\ref{eqn:pairwise_upper_bound}) and (\ref{eqn:pairwise_lower_bound_improved}) for all $m \times n$ Gaussian pairs.
\begin{corollary}
\label{crl:objective_function_bounds}
(Bounds of the objective function)
For the 3D transformation domain $C_r\times C_t$ centred at $(\br_0,\bt_0)$ with translation sphere radius $\rho$, the upper bound $\bar{f}$ and the lower bound $\ubar{f}$ of the optimal objective function value $f(\bR_\br,\bt)$ can be chosen as
\begin{align}
\bar{f}\left(\bR_{\br}, \bt \right) \!&=\! -\!\! \sum_{i = 1}^{m} \sum_{j = 1}^{n} \!\frac{\phi_{i}^{\mathcal{X}} \phi_{j}^{\mathcal{Y}}}{Z} \exp \!\! \left[ \!-\frac{\left[\bar{e}_{ij}\left(\bR_{\br}, \bt\right) \right]^2}{2 \left[ \sigma_{\scriptscriptstyle i\mathcal{X}}^2 + \sigma_{\scriptscriptstyle j \mathcal{Y}}^2 \right]} \!\right]
\label{eqn:objective_function_bounds_upper}\\
\ubar{f}\left(\bR_{\br}, \bt \right) \!&=\! -\!\! \sum_{i = 1}^{m} \sum_{j = 1}^{n}  \!\frac{\phi_{i}^{\mathcal{X}} \phi_{j}^{\mathcal{Y}}}{Z} \exp \!\! \left[ \!-\frac{\left[\ubar{e}_{ij}\left(\bR_{\br}, \bt\right) \right]^2}{2 \left[ \sigma_{\scriptscriptstyle i\mathcal{X}}^2 + \sigma_{\scriptscriptstyle j\mathcal{Y}}^2 \right]} \!\right] \!\!.
\label{eqn:objective_function_bounds_lower}
\end{align}
\end{corollary}

\section{The GOGMA Algorithm}
\label{sec:gogma}

The  Globally-Optimal Gaussian Mixture Alignment (GOGMA) algorithm is outlined in Algorithm~\ref{alg:gogma}. It employs branch-and-bound with depth-first search using a priority queue where the priority is inverse to the lower bound (Line~\ref{alg:priority_queue}). The algorithm terminates with $\epsilon$-optimality, whereby the difference between the best function value $f^*$ so far and the global lower bound $\ubar{f}$ is less than $\epsilon$ (Line~\ref{alg:stopping_criterion}).

In this implementation, the upper and lower bounds of 4096 sub-cubes are found simultaneously on the GPU (Line~\ref{alg:kernel}).
A higher branching factor can be used, although memory considerations must be taken into account to ensure that the priority queue does not increase much faster than it can be pruned.
A branching factor of 4096 performs well and does not require a high-end GPU. Other than bound calculation, the code is executed entirely on the CPU.

Lines~\ref{alg:gma_initialisation} and \ref{alg:gma} show how the local Gaussian Mixture Alignment (GMA) algorithm is integrated.
Firstly, the best-so-far function value $f^*$ and the associated transformation parameters are initialised using GMA (Line~\ref{alg:gma_initialisation}). Within the main loop, GMA is run whenever the BB algorithm finds a sub-cube $C_i$ that has an upper bound less than the best-so-far function value $f^*$ (Line~\ref{alg:gma_criterion}). GMA is initialised with $(\br_{0i}, \bt_{0i})$, the centre transformation of $C_i$. In this way, BB and GMA collaborate, with GMA quickly converging to the closest local minimum and BB guiding the search into the convergence basins of increasingly lower local minima. Hence, BB jumps the search out of local minima and GMA accelerates convergence by refining $f^*$. Importantly, the faster $f^*$ is refined, the more sub-cubes are discarded, since those with lower bounds higher than $f^*$ are culled (Line~\ref{alg:push_queue}).

The algorithm is designed in such a way that early termination outputs the best-so-far transformation. Hence, if a limit is set on the runtime, a best-guess transformation can be provided for those alignments that exceed the threshold. While $\epsilon$-optimality will not be guaranteed for them, in practise this is often adequate.
In view of this, and to accelerate the removal of redundant sub-cubes, Line~\ref{alg:gma_criterion} may be modified such that GMA is run for every sub-cube of the first subdivision and $f^*$ is updated with the best function value of that batch. We denote this as batch-initialised GOGMA.

\begin{algorithm}
\begin{algorithmic}[1]
\Require mixture models $\mathcal{G}_{\mathcal{X}}$ and $\mathcal{G}_{\mathcal{Y}}$, parametrised by means $\bx$ and $\by$ respectively, variances $\sigma^2$ and mixture weights $\phi$; threshold $\epsilon$; initial transformation hypercube $\mathcal{C} = \mathcal{C}_r \times \mathcal{C}_t$ centred at $(\br_0,\bt_0)$.

\Ensure $\epsilon$-optimal value $f^*$ and corresponding $\br^*$, $\bt^*$.

\State Run GMA: $(f^*, \br^*, \bt^*) \gets \text{GMA}(\br_0, \bt_0)$\label{alg:gma_initialisation}
\State Add hypercube $\mathcal{C}$ to priority queue $Q$
\Loop
\State Remove cube $C$ with lowest lower-bound $\ubar{f}$ from $Q$\label{alg:priority_queue}
\State \textbf{if} $f^* - \ubar{f} \leqslant \epsilon$ \textbf{then} Terminate \label{alg:stopping_criterion}
\State In parallel, evaluate $\bar{f}_i$ and $\ubar{f}_i$ for all sub-cubes of $C$\label{alg:kernel}
\ForAll{sub-cubes $C_i$}
\State \textbf{if} $\bar{f}_i < f^*$ \textbf{then} $(f^*, \br^*, \bt^*) \gets \text{GMA}(\br_{0i}, \bt_{0i})$\label{alg:gma}\label{alg:gma_criterion}
\State \textbf{if} $\ubar{f}_i < f^*$ \textbf{then} Add $C_i$ to queue: $Q \gets C_i$\label{alg:push_queue}
\EndFor
\EndLoop
\end{algorithmic}
\caption{GOGMA: An algorithm for globally-optimal Gaussian mixture alignment in $SE(3)$}
\label{alg:gogma}
\end{algorithm}

\vspace{-5pt}
\section{Experimental Results}
\label{sec:experimental_results}

GOGMA was evaluated using many different point-sets, including 3D data collected in the lab and in the field. In order to test the algorithms across a uniformly-distributed sample of $SO(3)$, the 72 base grid rotations from Incremental Successive Orthogonal Images (ISOI) \cite{yershova2009generating} were used. Translation perturbations were not applied, since centring and scaling the point-sets to $[-1,1]^3$ before conversion to GMM removes these perturbations.
The transformation domain was set to be $[-\pi, \pi]^3 \times [-0.5, 0.5]^3$. This corresponds to a minimum detectable bounding box overlap of \texttildelow$42\%$.

Except where otherwise specified, the convergence threshold was set to $\epsilon = 0.1$, the number of Gaussian components was set to $m,n \approx 50$, batch initialisation was used and the GMMs were Support Vector--parametrised Gaussian Mixtures (SVGMs), whereby an SVM and a mapping are used to efficiently construct an adaptive GMM from point-set data~\cite{campbell2015adaptive}.
SVGMs allow the user to specify the approximate number of components and set equal variances automatically, based on the desired number of components.

Although GOGMA is a general-purpose Gaussian mixture alignment algorithm, the runtime results include the time required for GMM construction, to facilitate comparison with other point-set registration algorithms. All experiments were run on a PC with a 3.7GHz Quad Core CPU with 32GB of RAM and a Nvidia GeForce GTX 980 GPU. The GOGMA code is written in unoptimised C++ and uses the VXL numerics library \cite{vxl2014} for local GMA optimisation.

It is crucial to observe that finding the global optimum does not necessarily imply finding the ground-truth transformation. For GMMs that describe partially overlapping point-sets, there may exist an alignment that produces a smaller function value than the ground-truth alignment. However, the $L_2$ density distance objective function is much less susceptible to partial overlap than others~\cite{campbell2015adaptive}, including the $L_2$ norm closest point error that is minimised by ICP.

\subsection{Optimality}
\label{sec:experimental_results_optimality}

To demonstrate optimality of the algorithm with respect to the objective function, the reconstructed \textsc{dragon}~\cite{curless2014dragon} and \textsc{bunny}~\cite{turk2014bunny} point-sets were aligned with transformed copies of themselves, using the 72 ISOI rotations. Identical point-sets were required to obtain the ground-truth optimal objective function values.
The global optimum was found for all $144$ registrations, with mean separations from the optimal value being $9 \!\times\! 10^{-8}$ and $3 \!\times\! 10^{-7}$ and mean runtimes being $17$s and $14$s for \textsc{dragon} and \textsc{bunny} respectively.
For batch initialisation, the mean runtimes were $33$s and $29$s. The evolution of the global upper and lower bounds is shown in Figure~\ref{fig:bound_evolution}. It can be seen that BB and GMA collaborate to reduce the upper bound: BB guides the search into the convergence basins of increasingly lower local minima and GMA refines the bound by jumping to the nearest local minimum. Discontinuities in the lower bound occur when an entire sub-cube level has been explored.
With batch initialisation, the global minimum is generally captured at the start of the algorithm. The remaining time is spent increasing the lower bound until $\epsilon$-optimality can be guaranteed.
While sometimes slower for simpler datasets or larger $\epsilon$, it usually reduces runtime and is the preferred setting.

\begin{figure}[!t]
\centering
\begin{subfigure}[]{0.495\columnwidth}
\includegraphics[trim=0 3.5pt 0 0pt,  clip=true, width=\columnwidth]{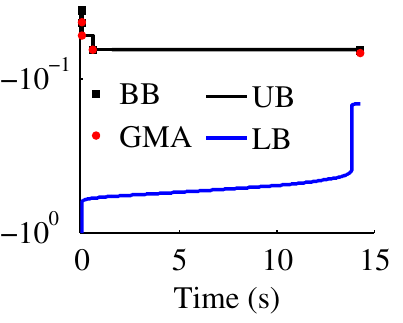}
\caption{\textsc{dragon}}
\label{fig:bound_evolution_dragon}
\end{subfigure}
\hfill
\begin{subfigure}[]{0.495\columnwidth}
\includegraphics[trim=0 3.5pt 0 0pt,  clip=true, width=\columnwidth]{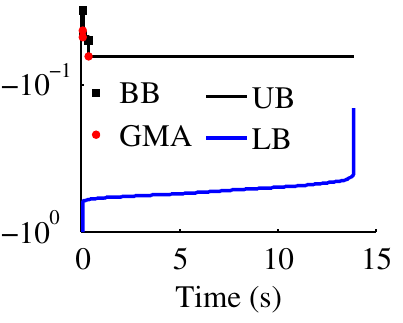}
\caption{\textsc{bunny}}
\label{fig:bound_evolution_bunny}
\end{subfigure}
\caption{Evolution of the upper and lower bounds for the reconstructed \textsc{dragon} and \textsc{bunny} models. The objective function value is plotted against time. Note the logarithmic scale.}
\label{fig:bound_evolution}
\end{figure}

\subsection{Fully-Overlapping Registration}
\label{sec:experimental_results_full_overlap}

In these experiments, we evaluated the performance of GOGMA by aligning single-view partial scans with a full 3D model. We used the \textsc{dragon} dataset, consisting of one reconstructed model (\textsc{dragon-recon}) and 15 partial scans (\textsc{dragon-stand}). The 72 base ISOI rotations were used as the initial transformations for the partial scans.
For the standard parameter settings, GOGMA found the correct alignment in all 1080 cases, as shown in Table~\ref{tab:results_dragon_gmm} (SVGM).

\begin{table}[!t]
\centering
\caption{Effect of GMM type on the accuracy and runtime of the GOGMA algorithm.
The mean/max translation error (in metres), rotation error (in degrees) and runtime (in seconds) are reported.}
\label{tab:results_dragon_gmm}
\newcolumntype{C}{>{\centering\arraybackslash}X}
\begin{tabularx}{\columnwidth}{l C C C}
\hline
GMM Type & SVGM & KDE & EM\\
\hline
Translation & \textbf{0.004/0.008} & 0.14/0.21 & 0.02/0.18\\
Rotation & \textbf{1.5/2.7} & 116/167 & 7.2/80\\
\hline
Runtime & 34/50 & \textbf{15/15} & 4960/4965\\ 
\hline
\end{tabularx}
\end{table}

To investigate the effect of other GMM types on the accuracy and runtime of the algorithm, we repeated the experiment with fixed-bandwidth Kernel Density Estimation (KDE)~\cite{jian2011robust} and Expectation Maximisation (EM)~\cite{dempster1977maximum}.
The number of components was fixed ($m,n = 50$), but the variances and mixture weights were set by the algorithms. For KDE, the variance was found by parameter search and the point-sets were randomly downsampled to $m$ points.
The results were poor, due to the small number of components imposed by GOGMA for tractability.
The EM results show that it is a suitable input to GOGMA in terms of alignment accuracy, however the implementation~\cite{figueiredo2002unsupervised} took $4\,663$s to process the model and \texttildelow$256$s to process each scan, making it impractical unless more efficiently implemented.
Considering both speed and accuracy, SVGMs are recommended.

To investigate the effect of other factors on the runtime, one was varied while the others were kept at the defaults.
The 72 ISOI rotations were applied to scan~0 and the mean runtimes were reported for standard and batch initialisations. The scan, aligned by GOGMA, is shown in Figure~\ref{fig:runtime_scan} in red.
The results for differing numbers of Gaussian components $m,n$ are shown in Figure~\ref{fig:runtime_n}. The quadratic shape reflects the $O(mn)$ per-iteration complexity.
The results for differing values of the convergence threshold $\epsilon$ are shown in Figure~\ref{fig:runtime_epsilon}. For values of $\epsilon$ close to zero, the runtime increases steeply, while larger values allow the algorithm to terminate quicker, albeit with a looser optimality guarantee. We found that $\epsilon \!=\! 0.1$ was a suitable default, having a 100\% success rate.
The runtime is also affected by the quality of the lower bound, as shown in Figure~\ref{fig:runtime_lb}. The GOGMA lower bound is tighter than the Go-ICP lower bound~\cite{yang2013goicp}, which uses the distance to an uncertainty sphere containing the spherical cap, as reflected in the runtimes.

\begin{figure}[!t]
\centering
\begin{subfigure}[]{0.495\columnwidth}
\includegraphics[trim=90pt 80pt 150pt 30pt,  clip=true, width=\columnwidth]{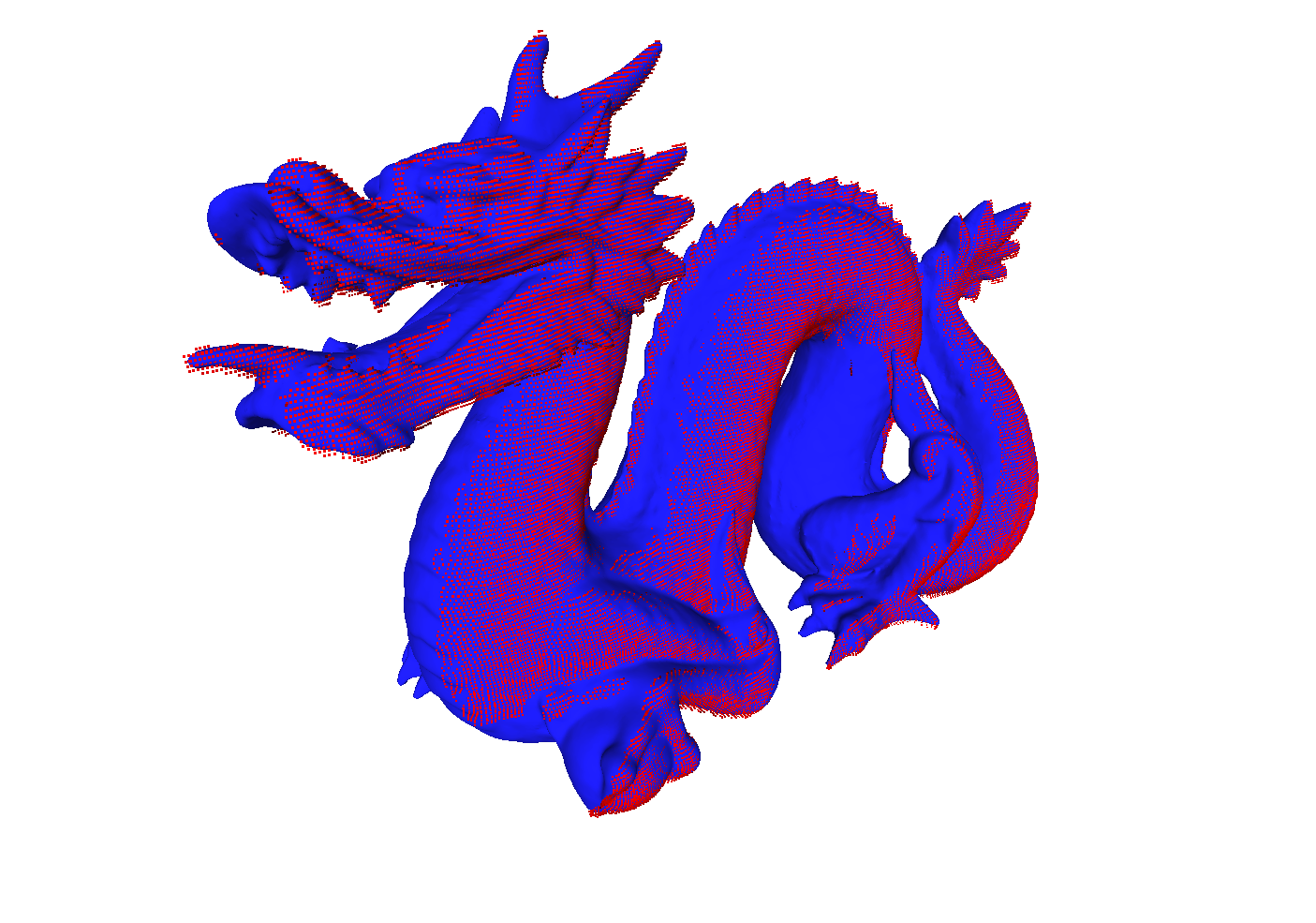}
\caption{Scan (in red) aligned to model}
\label{fig:runtime_scan}
\end{subfigure}
\hfill
\begin{subfigure}[]{0.495\columnwidth}
\includegraphics[trim=0 3pt 0 4pt,  clip=true, width=\columnwidth]{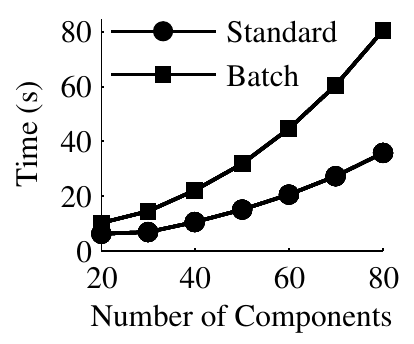}
\caption{Runtime versus $m,n$}
\label{fig:runtime_n}
\end{subfigure}

\begin{subfigure}[]{0.495\columnwidth}
\includegraphics[trim=0 3pt 0 4pt,  clip=true, width=\columnwidth]{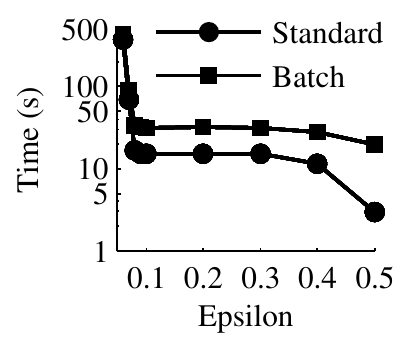}
\caption{Runtime versus $\epsilon$}
\label{fig:runtime_epsilon}
\end{subfigure}
\hfill
\begin{subfigure}[]{0.495\columnwidth}
\includegraphics[trim=0 3pt 0 4pt,  clip=true, width=\columnwidth]{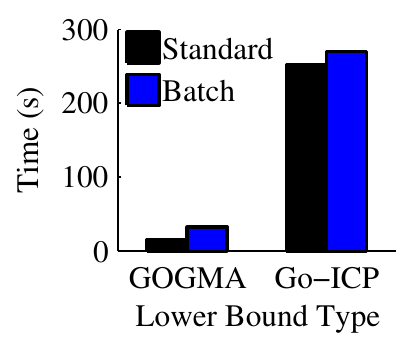}
\caption{Runtime versus lower bound}
\label{fig:runtime_lb}
\end{subfigure}
\caption{Mean runtime of GOGMA on the \textsc{dragon} dataset with respect to different factors, for the alignment of \textsc{dragon-recon} with point-set 0 of \textsc{dragon-stand}, transformed by 72 uniformly distributed rotations. Note the logarithmic scale in (\subref{fig:runtime_epsilon}).}
\label{fig:runtime}
\end{figure}

\subsection{Partially-Overlapping Registration}
\label{sec:experimental_results_partial_overlap}

To evaluate the performance of GOGMA on large-scale field datasets, we used the two challenging laser scanner datasets in Table~\ref{tab:results_large-scale_datasets} \cite{pomerleau2012challenging}. \textsc{stairs} is a structured in/outdoor dataset with large and rapid variations in scanned volumes. \textsc{wood-summer} is an unstructured outdoor dataset with dynamic objects.
The symmetric inlier fraction was used to calculate the overlap: the fraction of points from the joint set within $10\bar{d}$ of a point from the other point-set, where $\bar{d}$ is the mean closest point distance.
Sequential point-sets were aligned using GOGMA with/out refinement, Go-ICP \cite{yang2013goicp}, ICP \cite{besl1992method} and CPD \cite{myronenko2010point} with the 72 ISOI rotations as initial transformations.
GOGMA was refined by running GMA from the output transformation with $m,n \!\approx\! 1000$ components.
The coarse, medium and fine registration inlier rates are defined as the fraction of alignments with translation and rotation errors less than $2$m/$10^{\circ}$, $1$m/$5^{\circ}$ and $0.5$m/$2.5^{\circ}$ respectively.
Tables~\ref{tab:results_stairs} and~\ref{tab:results_wood_summer} and Figure~\ref{fig:results_qualitative} show the results.

\begin{table}[!t]
\centering
\caption{Characteristics of the large-scale field datasets \cite{pomerleau2012challenging}.}
\label{tab:results_large-scale_datasets}
\newcolumntype{C}{>{\centering\arraybackslash}X}
\begin{tabularx}{\columnwidth}{l C C C}
\hline
Dataset & \textsc{stairs} & \textsc{wood-summer}\\
\hline
\# Point-Sets & 31 & 37\\
Mean \# Points & 191\,000 & 182\,000\\
Mean Overlap & 76\% & 77\%\\
\hline
\end{tabularx}
\end{table}

\begin{table}[!t]
\centering
\caption{
Alignment results for \textsc{stairs}. The mean translation error (in metres), rotation error (in degrees), coarse (C), medium (M) and fine (F) registration inlier rates (defined in the text) and mean runtime (in seconds) are reported. GOGMA is denoted by [*], GOGMA with refinement by [*]$_{\text{+}}$, Go-ICP with $\epsilon \!=\!10^{-4}$ by \cite{yang2013goicp}$_{\text{a}}$, Go-ICP with $\epsilon \!=\! 5 \!\times\! 10^{-5}$ by \cite{yang2013goicp}$_{\text{b}}$, ICP by \cite{besl1992method} and CPD by \cite{myronenko2010point}.
}
\label{tab:results_stairs}
\newcolumntype{C}{>{\centering\arraybackslash}X}
\begin{tabularx}{\columnwidth}{l C C C C C C}
\hline
Method & [*] & [*]$_{\text{+}}$ & \cite{yang2013goicp}$_{\text{a}}$ & \cite{yang2013goicp}$_{\text{b}}$ & \cite{besl1992method} & \cite{myronenko2010point}\\
\hline
Translation & 0.26 & \textbf{0.04} & 1.63 & 1.17 & 4.67 & 5.24\\
Rotation & 1.25 & \textbf{0.32} & 30.9 & 19.4 & 107 & 88.8\\
\hline
Inlier \% (C) & \textbf{100} & \textbf{100} & 71.8 & 80.9 & 15.5 & 38.8\\ 
Inlier \% (M) & \textbf{100} & \textbf{100} & 48.5 & 51.9 & 13.4 & 28.6\\ 
Inlier \% (F) & 80.0 & \textbf{99.7} & 19.6 & 21.2 & 6.5 & 7.1\\ 
\hline
Runtime & 49.6 & 71.2 & 31.6 & 103 & \textbf{0.38} & 4.2\\ 
\hline
\end{tabularx}
\end{table}

\begin{table}[!t]
\centering
\caption{Alignment results for \textsc{wood-summer}.
}
\label{tab:results_wood_summer}
\newcolumntype{C}{>{\centering\arraybackslash}X}
\begin{tabularx}{\columnwidth}{l C C C C C C}
\hline
Method & [*] & [*]$_{\text{+}}$ & \cite{yang2013goicp}$_{\text{a}}$ & \cite{yang2013goicp}$_{\text{b}}$ & \cite{besl1992method} & \cite{myronenko2010point}\\
\hline
Translation & 0.72 & \textbf{0.13} & 1.33 & 0.69 & 7.37 & 8.13\\
Rotation & 3.09 & \textbf{0.68} & 9.66 & 5.19 & 109 & 90.7\\
\hline
Inlier \% (C) & \textbf{100} & \textbf{100} & 78.2 & 84.1 & 11.3 & 39.5\\ 
Inlier \% (M) & 75.0 & \textbf{99.9} & 36.6 & 64.5 & 10.8 & 19.3\\ 
Inlier \% (F) & 16.7 & \textbf{99.9} & 13.2 & 27.5 & 5.4 & 0.8\\ 
\hline
Runtime & 29.5 & 49.6 & 26.2 & 77.7 & \textbf{0.44} & 4.2\\ 
\hline
\end{tabularx}
\end{table}

\begin{figure}[!t]
\centering
\begin{subfigure}[]{0.495\columnwidth}
\includegraphics[trim=260pt 150pt 300pt 100pt,  clip=true, width=\columnwidth]{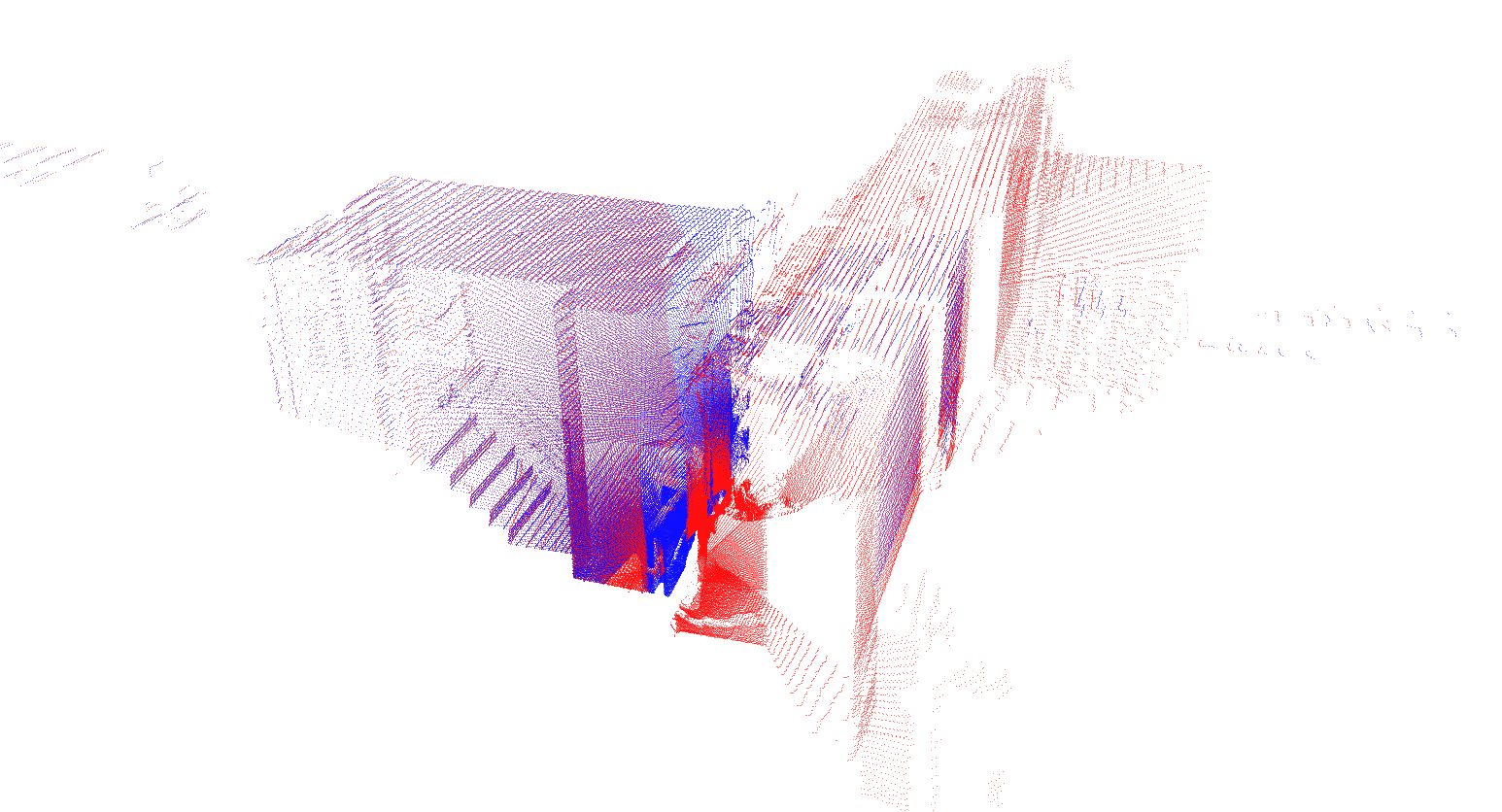}
\caption{\textsc{stairs} point-sets 4 \& 5}
\label{fig:results_qualitative_stairs}
\end{subfigure}
\hfill
\begin{subfigure}[]{0.495\columnwidth}
\includegraphics[trim=340pt 170pt 260pt 100pt,  clip=true, width=\columnwidth]{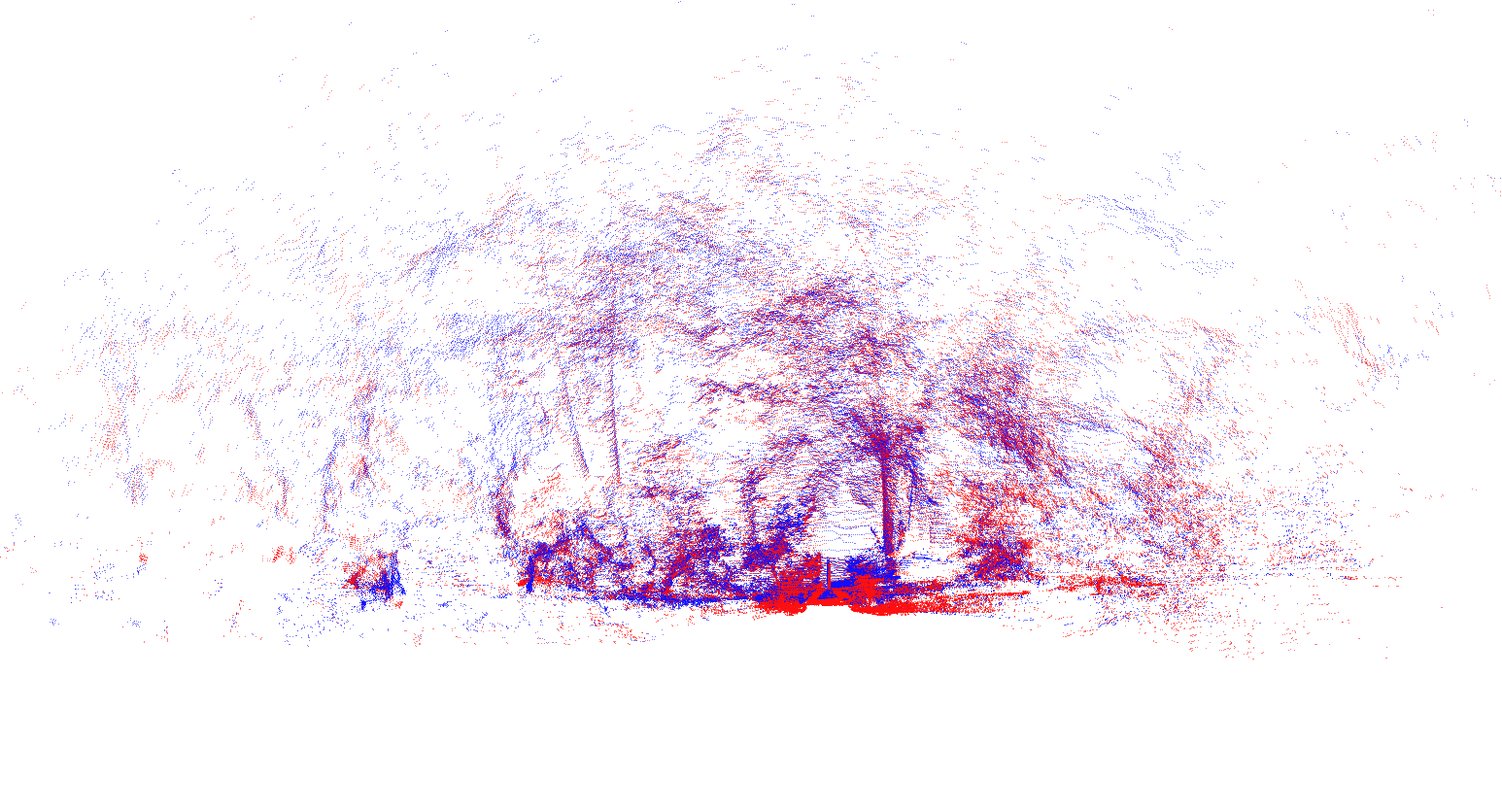}
\caption{\textsc{wood-summer} point-sets 0 \& 1}
\label{fig:results_qualitative_wood_summer}
\end{subfigure}
\caption{Qualitative results for two large-scale datasets. The blue scan was aligned by GOGMA from an arbitrary initial pose against the red scan, followed by GMA refinement. Best viewed in colour.}
\label{fig:results_qualitative}
\end{figure}

GOGMA significantly outperformed Go-ICP in these experiments, finding the correct transformation in all cases under the coarse criterion. Crucially, the success of the refinement step in finding the correct transformation in virtually all cases under the fine criterion indicates that GOGMA found the correct alignment, up to the granularity of the $50$ component representation.
Go-ICP performed poorly with the loose convergence threshold $\epsilon \!=\! 10^{-4}$ and $N \!=\! 50$ points. With $\epsilon$ an order of magnitude smaller ($10^{-5}$), $N$ an order of magnitude greater ($500$), or any trimming, the runtime became prohibitively slow.
The tightest feasible $\epsilon$ ($5 \!\times\! 10^{-5}$) failed to coarsely align $19\%$ of cases for \textsc{stairs} and $16\%$ for \textsc{wood-summer}.
Finally, the results show that ICP and CPD both perform poorly without a good pose prior, converging to local minima for most initialisations.

A specific application is the kidnapped robot problem: finding the pose of a sensor within a 3D map. 
We perturbed 4 scans of different rooms in the \textsc{apartment} dataset \cite{pomerleau2012challenging} by the 72 ISOI rotations to simulate being lost and used GOGMA to localise the scans within the map. As shown in Table~\ref{tab:results_sensor_localisation} and Figure \ref{fig:results_sensor_localisation}, all positions were correctly localised.

\begin{table}[!t]
\centering
\caption{Sensor localisation results for \textsc{apartment}.
The mean translation error (in metres), rotation error (in degrees), runtime (in seconds) and fine (F) registration inlier rates are reported.}
\label{tab:results_sensor_localisation}
\newcolumntype{C}{>{\centering\arraybackslash}X}
\begin{tabularx}{\columnwidth}{l C C C C}
\hline
Room Scan & A & B & C & D\\
\hline
Translation & 0.16 & 0.22 & 0.40 & 0.35\\
Rotation & 0.93 & 0.89 & 1.95 & 2.35\\
Inlier \% (F) & 100 & 100 & 100 & 100\\
\hline
Runtime & 328 & 383 & 379 & 409\\
\hline
\end{tabularx}
\end{table}

\begin{figure}[!t]
\centering
\begin{minipage}[c]{0.22\columnwidth}
\centering
\includegraphics[trim=230pt 60pt 500pt 0pt,  clip=true, width=\columnwidth]{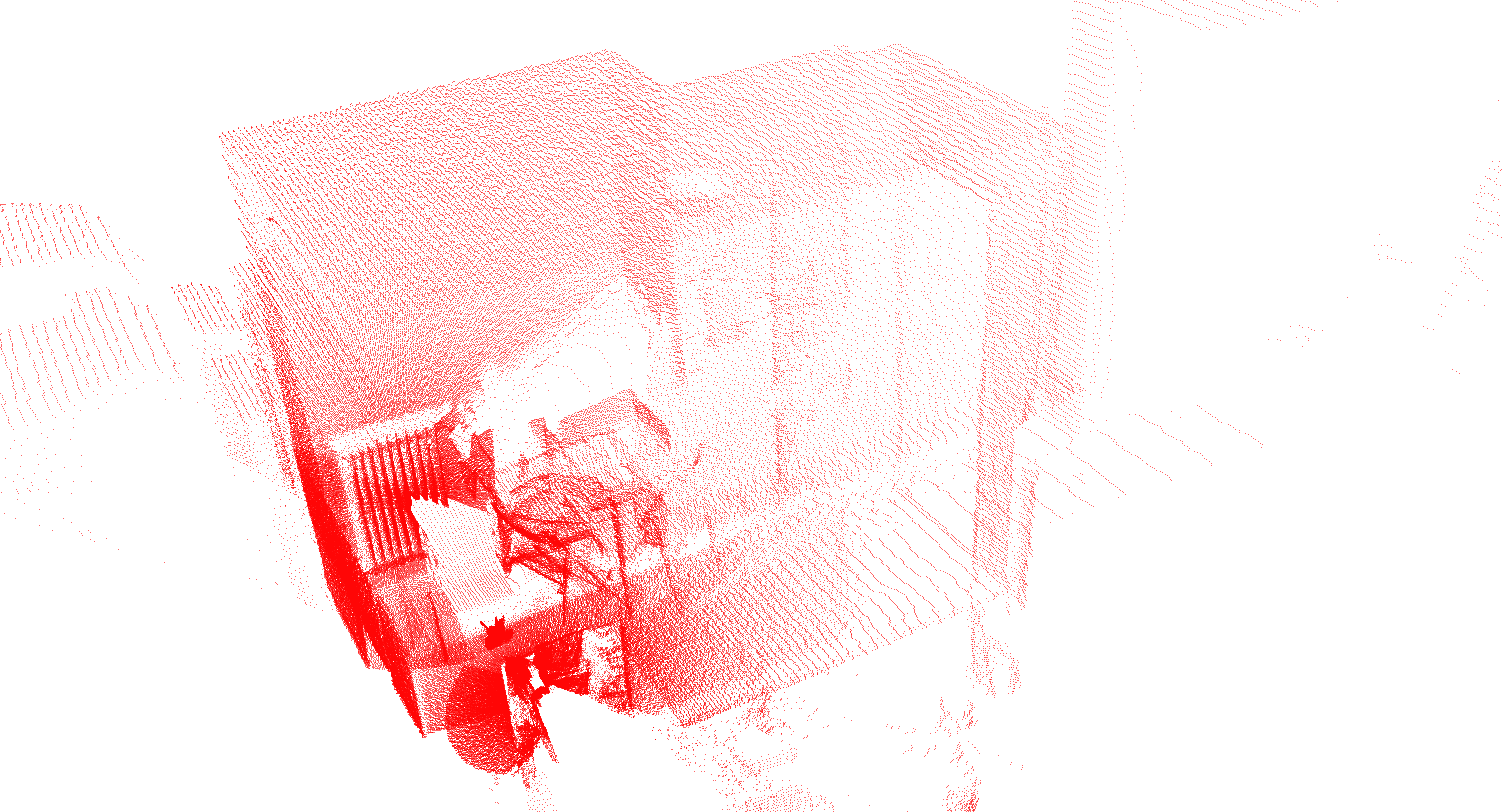}
\par\vfill
\includegraphics[trim=500pt 100pt 360pt 180pt,  clip=true, width=\columnwidth]{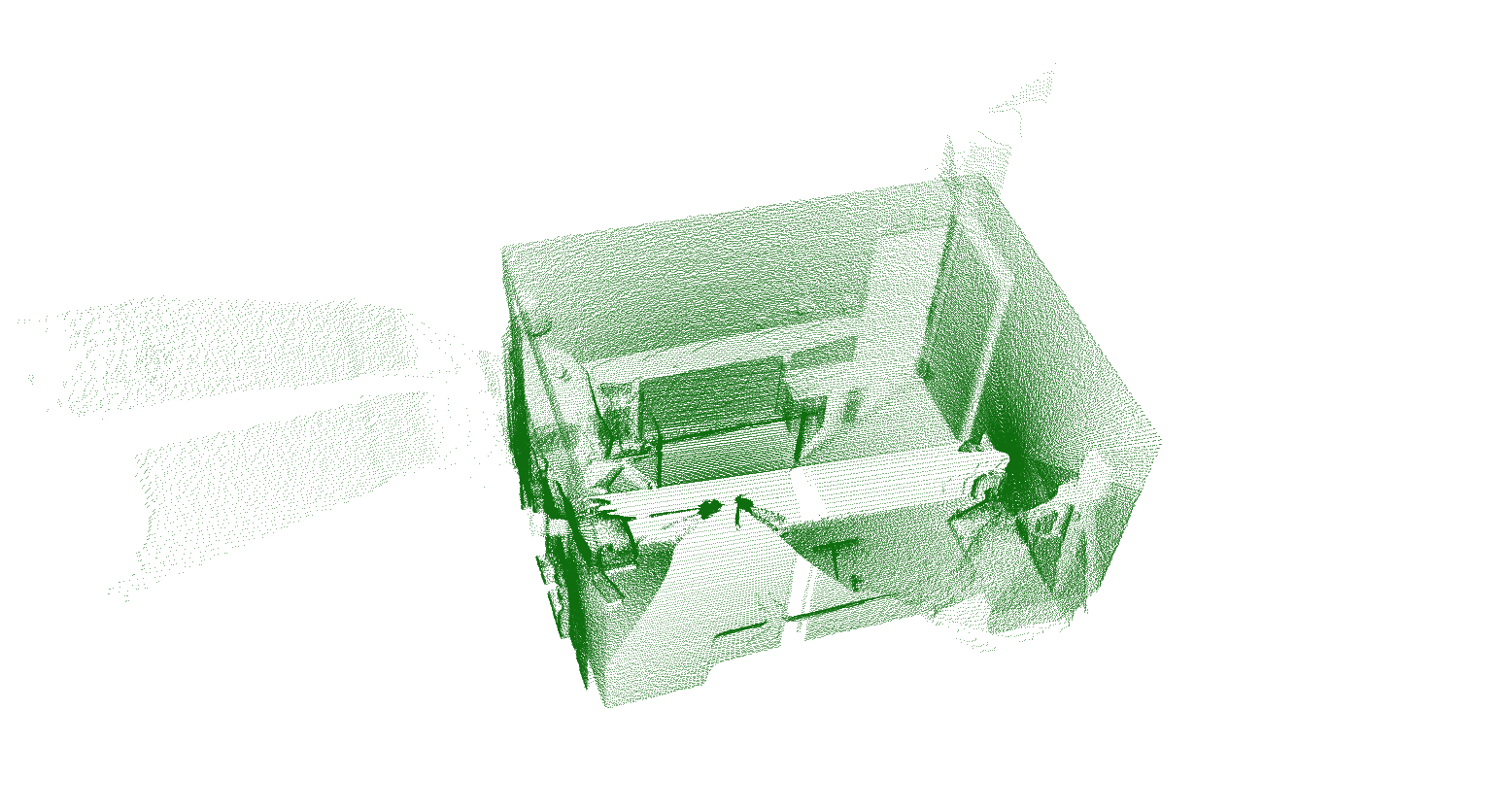}
\end{minipage}
\hfill
\begin{minipage}[c]{0.03\columnwidth}
\centering
\textbf{A}
\vspace{20pt}\\
\textbf{C}
\end{minipage}
\hfill
\begin{minipage}[c]{0.42\columnwidth}
\centering
\includegraphics[trim=300pt 60pt 490pt 100pt,  clip=true, width=\columnwidth]{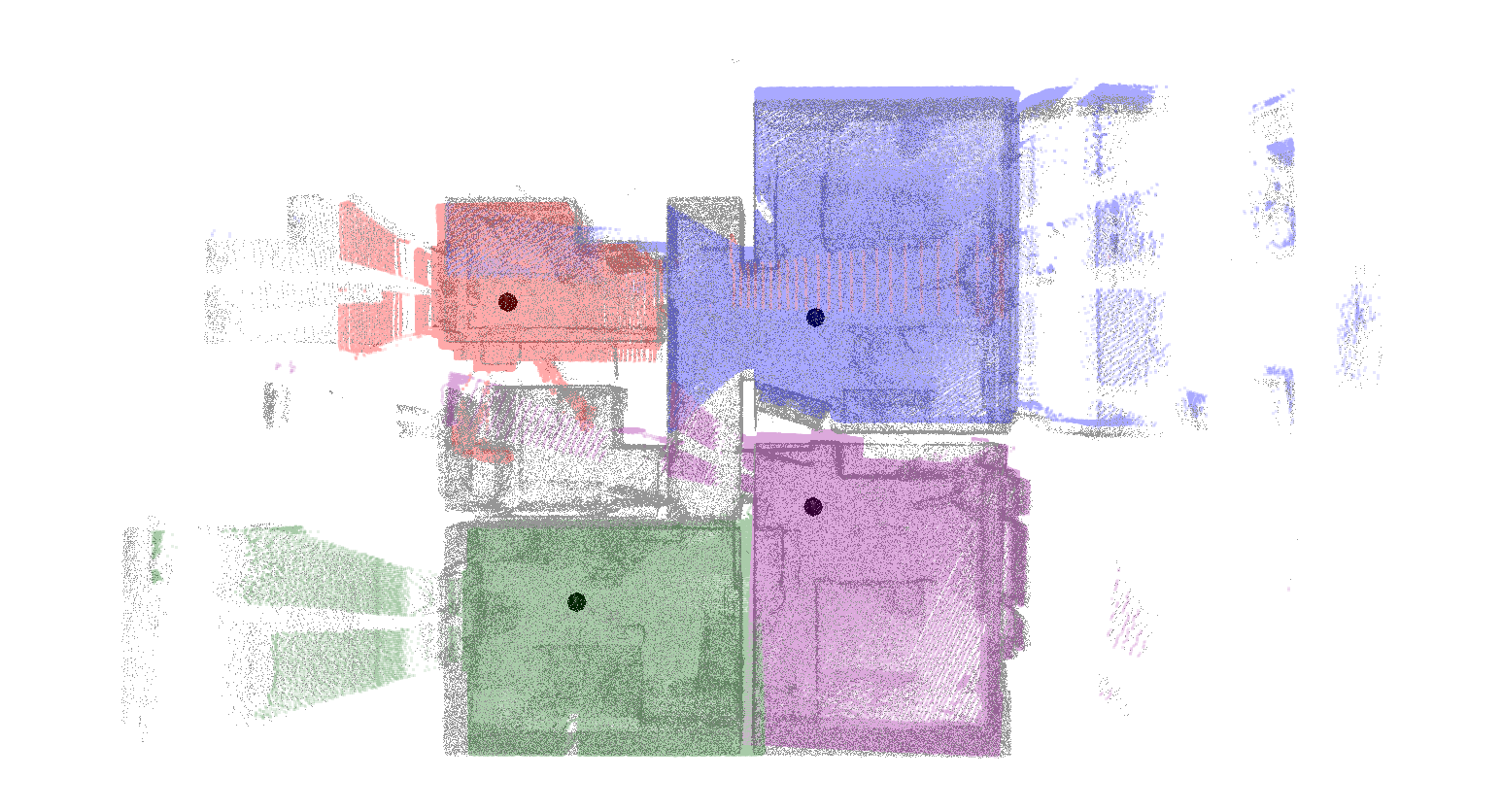}
\end{minipage}
\hfill
\begin{minipage}[c]{0.03\columnwidth}
\centering
\textbf{B}
\vspace{20pt}\\
\textbf{D}
\end{minipage}
\hfill
\begin{minipage}[c]{0.22\columnwidth}
\centering
\includegraphics[trim=400pt 200pt 530pt 80pt,  clip=true, width=\columnwidth]{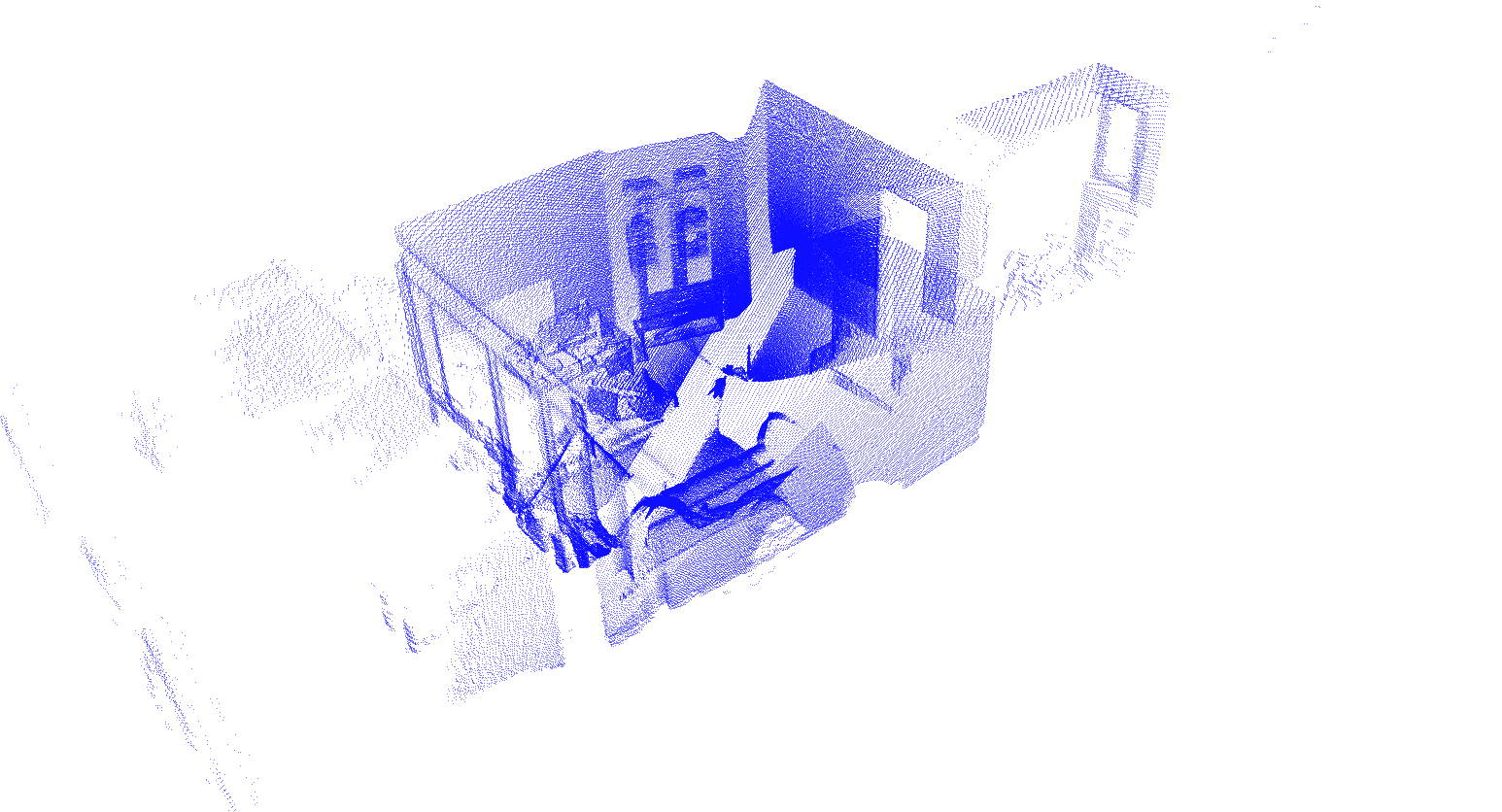}
\par\vfill
\includegraphics[trim=230pt 10pt 420pt 100pt,  clip=true, width=\columnwidth]{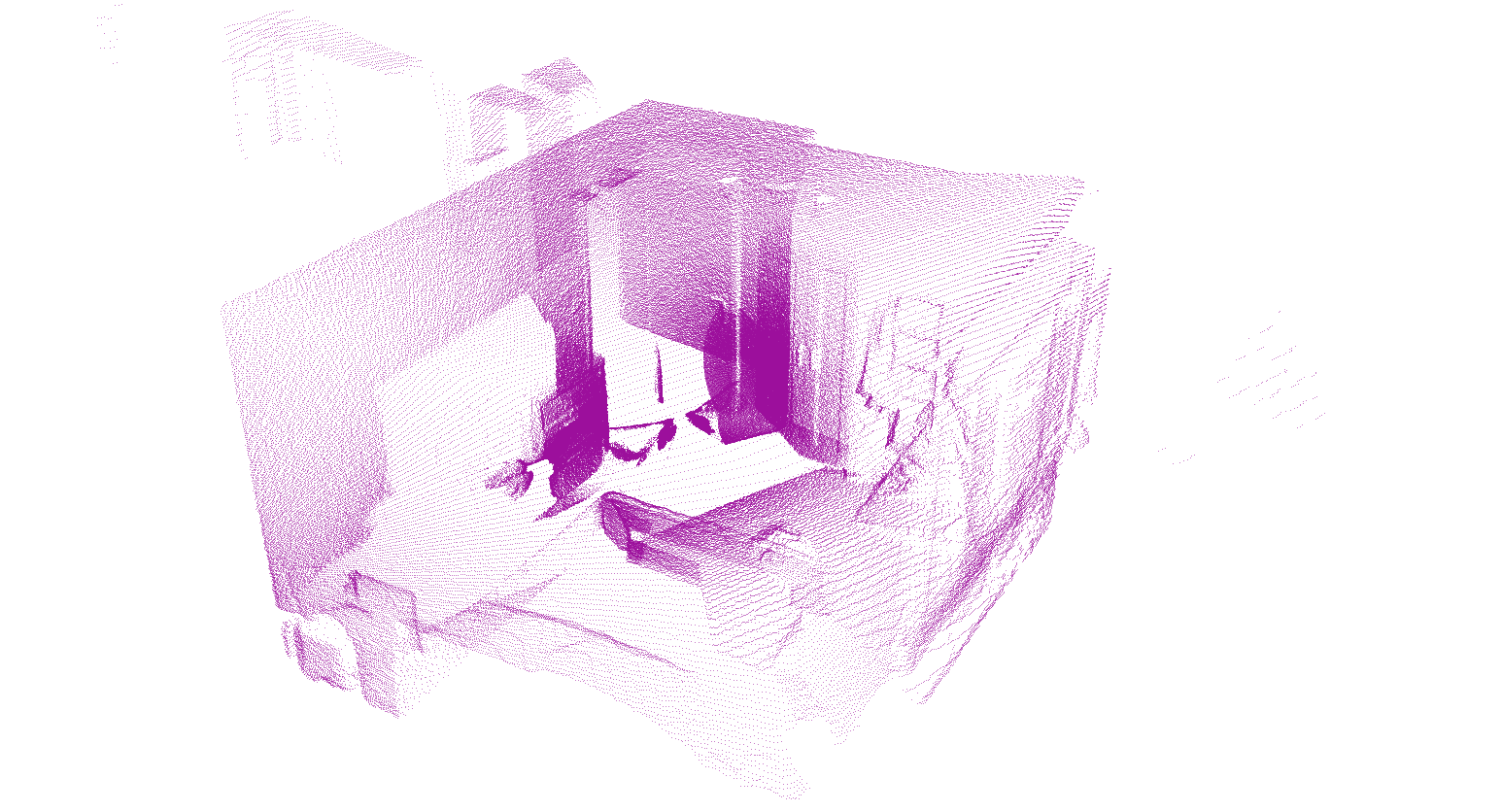}
\end{minipage}
\caption{Coarse pose estimates (black spheres) of the sensor locations for 4 room scans (red, blue, green and purple) found by aligning each scan with the entire map (grey) using GOGMA.}
\label{fig:results_sensor_localisation}
\end{figure}

\section{Conclusion}
\label{sec:conclusion}

In this paper, we have introduced a globally-optimal solution to 3D Gaussian mixture alignment under the $L_2$ distance, with an application to point-set registration. The method applies the branch-and-bound algorithm to guarantee global optimality regardless of initialisation and uses local optimisation to accelerate convergence. The pivotal contribution is the derivation of the objective function bounds using the geometry of $SE(3)$.
The algorithm performed very well on challenging field datasets, due to an objective function that is robust to outliers induced by partial-overlap and occlusion.
%
There are several areas that warrant further investigation.
Firstly, runtime benefits could be realised by implementing the local optimisation on the GPU.
Dynamic branching factors would allow more parallelism for the same memory requirements.
Finally, extending the lower bound to handle full covariances would enable the algorithm to be applied to more general Gaussian mixtures.

{\small
\bibliographystyle{ieee}
\bibliography{campbell2016gogma}
}

\end{document}